\documentclass[letterpaper, 10 pt, conference]{ieeeconf}  

\IEEEoverridecommandlockouts                              

\usepackage{amsfonts, dsfont, mathtools, amsmath, amssymb, bbold, mathrsfs, bm, upgreek}
\usepackage{graphicx, float, epsfig, color, psfrag, adjustbox, enumerate}
\usepackage[font=small]{caption}
\usepackage{refcount, url, cleveref, soul}
\usepackage[noadjust]{cite}
\usepackage[usenames,dvipsnames,svgnames,table]{xcolor}
\usepackage{algorithm, algpseudocode}
\usepackage{booktabs}

\algrenewcommand\alglinenumber[1]{#1:}

\title{\LARGE \bf
A bioinspired internal model-based online estimator for planar pursuit
}

\author{Tengyue Liu$^{1}$, Xincheng Li$^{1}$, Sofia Morales Ferreira$^{1}$, Kevin S. Galloway$^{2}$, Udit Halder$^{1}$
\thanks{$^{1}$Department of Mechanical and Aerospace Engineering, University of South Florida, Tampa, FL, $^{2}$Department of Electrical and Computer Engineering, United States Naval Academy, Annapolis, MD.
  Corresponding e-mail:  {\tt\small udithalder@usf.edu}}%
}
\def\R{{\mathds{R}}}

\def\0{{\mathbb{0}}}
\def\1{{\mathds{1}}}

\newcommand{\norm}[1]{\left\lVert#1\right\rVert}

\definecolor{db}{RGB}{23,20,119}
\definecolor{dg}{RGB}{2,101,15}

\newtheorem{proposition}{Proposition}[section]

\newtheorem{remark}{Remark}

\usepackage{upgreek}
\newcommand{\dif}{\mathrm{d}}

\newcommand{\material}[1]{
	\ifthenelse{\equal{#1}{\kappa}}{\upkappa}{
	\ifthenelse{\equal{#1}{\nu}}{\upnu}{
	\ifthenelse{\equal{#1}{\omega}}{\upomega}{
	\ifthenelse{\equal{#1}{\sigma}}{\upsigma}{
	\ifthenelse{\equal{#1}{\theta}}{\uptheta}{
	\mathsf{#1}}}}}}
}

\graphicspath{{figures/}}

\begin{document}
\bstctlcite{BSTcontrol} 
\maketitle
\thispagestyle{empty}
\pagestyle{empty}


\begin{abstract}
Bioinspired feedback controls for pursuit, tracking, and collective motion are often expressed in terms of the relative configuration between interacting agents. In practice, however, onboard sensors may not directly provide all quantities required for feedback control, necessitating estimation of unobserved quantities. 
This paper develops a bioinspired internal model-based estimator for reconstructing those quantities from partial sensory observations and known self-motion.
State reconstruction is posed as an optimization problem that treats the relative kinematics as constraints and minimizes the disagreement between the internal model outputs and measurements from onboard sensors.  
Pontryagin's Maximum Principle is used to derive the necessary optimality conditions. A forward-backward algorithm is used to provide a numerical solution and a moving horizon formulation is employed for online implementation. The estimator is evaluated numerically against classical state estimators. 
Real-time implementation of the proposed framework on robotic hardware is demonstrated through two pursuit strategies.
\end{abstract}


\section{Introduction} \label{sec:intro}
Pursuit and interception of prey are prominently observed in natural predation, including that of falcons, dragonflies, and bats~\cite{mizutani2003motion, ghose2006echolocating, brighton2017terminal}. Analysis of the predator trajectories has revealed remarkable strategies employed by these animals. Examples include constant bearing strategy that maintains an approximately fixed inertial line of sight to the prey~\cite{galloway2013symmetry} and motion camouflage, where the predator moves so as to reduce their apparent transverse motion with respect to the prey~\cite{glendinning2004mathematics}.

These biological pursuit strategies have inspired a substantial body of work on engineered pursuit and interception systems, including motion camouflage~\cite{justh2006steering, mischiati2012dynamics, halder2015biomimetic}, constant bearing pursuit~\cite{galloway2013symmetry}, beacon-referenced pursuit~\cite{halder2016steering}, and leader-follower tracking~\cite{li2026feedback}. In these approaches, feedback control is constructed using geometric quantities such as relative position,
bearing, and prey turning rate, with the underlying assumption that these quantities are available to the pursuer either directly through onboard sensing or through communication. In practical settings, however, this assumption may not hold, requiring the necessary quantities to be reconstructed from the available measurements.

This reconstruction problem is closely related to nonlinear state estimation. Classical approaches such as the extended Kalman filter propagate a process model and correct the resulting prediction using incoming measurements \cite{ljung1979filter}, while geometric approaches exploit symmetries of the configuration space, including nonlinear complementary filters \cite{mahony2008nonlinear}, symmetry-preserving observers \cite{bonnabel2008symmetry}, invariant extended Kalman filters \cite{barrau2017invariant}, and equivariant observers \cite{serrano2025bearing}. Unknown pursuee controls may be included in such estimators through state augmentation, but doing so requires an assumed evolution model or prior. An alternative is to reconstruct the state trajectory through dynamic optimization, leading to moving horizon estimators \cite{robertson1996moving,rao2003constrained}. 

The approach taken in this paper to address the estimation problem is motivated by another idea from biological sensorimotor control: the use of {\it internal models}. The internal model hypothesis proposes that nervous systems construct internal representations of the body and its environment that are used to interpret sensory information in the context of predicted motion and guide own behavior accordingly~\cite{wolpert1995internal, mcnamee2019internal, huang2018internal}. Motivated by this idea, we construct within the pursuing agent an internal model of its target. 
The known relative kinematics constrain the evolution of this model, allowing the reconstruction problem to be posed as an optimization problem in which the internal trajectory is optimized to best explain the measurements. The framework is presented in the planar setting for simplicity.

\smallskip
The key contributions of this paper are as follows:
\smallskip

\noindent
1) This paper proposes an internal model-based estimator (IME) for the pursuer agent to reconstruct the full relative configuration and unknown motion commands of its target agent using partial, intermittent observations such as distance and bearing. The estimation problem is posed as an optimization problem whose necessary optimality conditions are obtained using Pontryagin's Maximum Principle (PMP)~\cite{liberzon2011calculus}. A major difference with the existing methods is IME's ability to estimate the motion controls (linear and angular speeds) of its \textit{freely} moving target.

\smallskip
\noindent
2) An iterative algorithm is proposed to solve the PMP conditions leading to optimal estimates. 
Furthermore, a moving-horizon implementation is employed for carrying out the estimation in an online fashion. 


\smallskip
\noindent
3) The efficacy of using the IME estimates for real-time implementation of 
two pursuit strategies (constant bearing and mutual motion camouflage) is demonstrated through numerical simulations and robotic experiments. 

The remainder of the paper is organized as follows: mathematical preliminaries are given in Sec.~\ref{sec:model}, our proposed estimator is presented in Sec.~\ref{sec:estimator}, numerical and robotic experiments are described in Sec.~\ref{sec:results}, and the paper is concluded in Sec.~\ref{sec:conclusion}.

\section{Preliminaries}
\label{sec:model}

\begin{figure}[t]
    \centering
    \includegraphics[width=.48\textwidth]{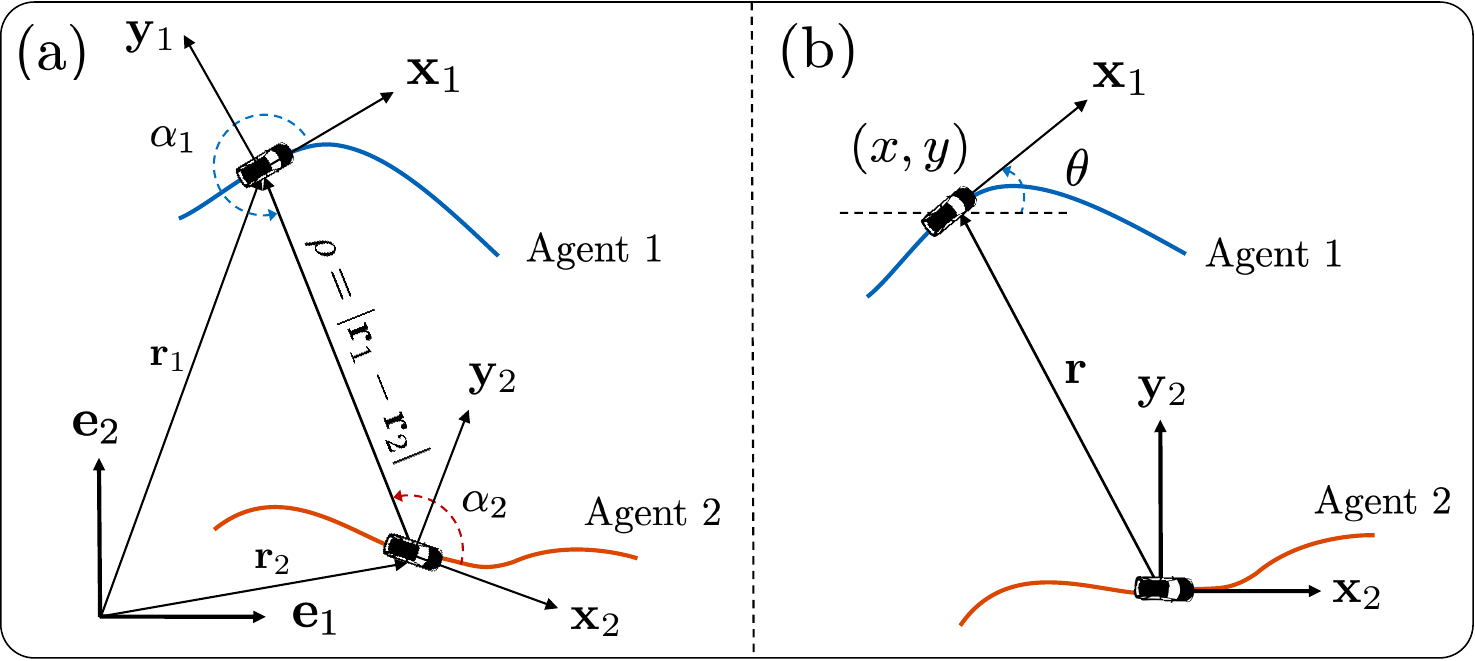}
    \caption{Two-agent formation. (a) The agents evolve in the global inertial frame $\{\mathbf e_1,\mathbf e_2\}$, with their relative configuration described by the shape variables $(\rho,\alpha_1,\alpha_2)$. (b) In the body-fixed frame $\{\mathbf x_2,\mathbf y_2\}$ attached to agent 2, the relative configuration of agent~1 is described by the coordinates $(x,y,\theta)$.}
    \vspace*{-15pt}
    \label{fig:formation}
\end{figure}

\subsection{Agent model}

Consider two agents moving in the plane spanned by the global inertial frame $\{\mathbf e_1,\mathbf e_2\}$ (see Fig.~\ref{fig:formation}(a)). Agent $i$ has global position $\mathbf r_i= x_i \, \mathbf{e}_1 + y_i \, \mathbf{e}_2$   
and heading $\theta_i$ that describes the moving frames $\mathbf x_i = \cos\theta_i \, \mathbf{e}_1 + \sin\theta_i \, \mathbf{e}_2$ and $\mathbf y_i = -\sin\theta_i \, \mathbf{e}_1 + \cos\theta_i \, \mathbf{e}_2$.  
We model each agent as a unicycle~\cite{bishop1975there, lavalle2006planning} as follows
\begin{equation}
\begin{aligned}
    \dot x_i = v_i\cos\theta_i, ~~
    \dot y_i = v_i\sin\theta_i, ~~
    \dot\theta_i = u_i, ~~ i \in \{1, 2\}
\end{aligned}
\label{eq:unicycle_dynamics}
\end{equation}
where $v_i(t)$ and $u_i(t)$ denote the translational and angular speeds (control inputs to the agents), respectively.

Unless otherwise specified, we will consider agent 2 to be the pursuer and agent 1 the pursuee. 


\subsection{Relative motion (plant model)}


For various pursuit strategies, the absolute position and orientation of the agents are often less important than their configuration relative to one another~\cite{galloway2013symmetry, halder2016steering, justh2006steering}. Therefore, let us denote the position and heading of the pursuee in the body-fixed frame of the pursuer by $\mathbf{r} :=  x \, \mathbf{x}_2 + y \, \mathbf{y}_2 = \mathbf{r}_{1} - \mathbf{r}_2$  
and $\theta$, respectively (see Fig.~\ref{fig:formation}(b)). 
We calculate
\begin{equation}
    \begin{bmatrix}
        x\\
        y
    \end{bmatrix}
    =
    \begin{bmatrix*}[r]
        \cos\theta_2 & \sin\theta_2\\
        -\sin\theta_2 & \cos\theta_2
    \end{bmatrix*}
    \begin{bmatrix}
        x_1-x_2\\
        y_1-y_2
    \end{bmatrix},
    \quad
    \theta=\theta_1-\theta_2
\label{eq:relative_coordinates}
\end{equation}

\noindent
{\it Relative state dynamics.}
Let us denote the 
\textit{relative state} as $z:=[x, \; y, \; \theta]^\top \in \R^3$, whose dynamics are obtained by taking derivatives of~\eqref{eq:relative_coordinates} and using \eqref{eq:unicycle_dynamics}
\begin{equation}
    \dot z = f(z, \eta_1, \eta_2) :=
    \begin{bmatrix}
        v_1\cos\theta-v_2+u_2y\\
        v_1\sin\theta-u_2x\\
        u_1-u_2
    \end{bmatrix}
\label{eq:relative_dynamics}
\end{equation}
where $\eta_i = [v_i, \; u_i]^\top$ represents the controls of agent~$i \in \{1, 2\}$. The dynamics~\eqref{eq:relative_dynamics} are regarded as the plant model (ground truth).

\noindent
{\it Observation model.}
We assume that the pursuer has access only to the relative distance and bearing to the pursuee. Such information may be inferred from onboard sensors, e.g., LiDARs. Therefore, we may treat the relative position $(x,y)$ as the output or observation of the plant model
\begin{equation}
    \mathsf{y}=h(z)=
    \begin{bmatrix}
        x\\
        y
    \end{bmatrix}
\label{eq:continuous_output}
\end{equation}

\begin{remark}
    The relative coordinates $(x,y,\theta)$ may also be interpreted as the coordinates of a relative rigid-body transformation on the special Euclidean group $SE(2)$. Let $ g_i = \left[\begin{smallmatrix} R(\theta_i) & \mathbf r_i\\
        0 & 1
    \end{smallmatrix}\right]
    \in SE(2)$ 
denote the configuration of agent $i$. Then the configuration of agent 1 as observed from agent 2 is
\begin{equation}
    g
    =
    g_2^{-1}g_1
    =
    \begin{bmatrix}
        R(\theta) & \begin{pmatrix}x \\ y\end{pmatrix} \\
        0 & 1
    \end{bmatrix}
\label{eq:relative_group_element}
\end{equation}
so that $(x,y,\theta)$ are precisely the coordinates of the relative group element $g$. This equivalent Lie-group interpretation is often used in geometric observers \cite{mahony2008nonlinear, bonnabel2008symmetry}.
\end{remark}

\subsection{Practical limitations of implementing feedback control}
An array of pursuit strategies~\cite{justh2006steering, galloway2013symmetry, halder2016steering, li2026feedback} can be expressed in a feedback form
\begin{align}
    \eta_2 = \eta_{\text{pursuit}}(z, \eta_1)
    \label{eq:pursuit_general}
\end{align}
Such feedback laws 
possess an important practical limitation: the quantities required by the controller might not all be directly available to an individual agent. For example, from the point of view of the pursuer, the heading of the pursuee ($\theta$) and its speed and steering controls ($\eta_1$) are generally unknown. 
Therefore, proper execution of any pursuit strategy of the form~\eqref{eq:pursuit_general}
requires the unknown quantities to be 
systematically estimated from observed relative motion.


\section{Internal Model-based Estimator (IME)}
\label{sec:estimator}
In this section, we propose an internal model-based estimator (IME) which uses partial observations to reconstruct the relative states and unknown controls of the pursuee required for feedback.
\subsection{Internal model and estimation}
We suppose that pursuer maintains an internal representation of the relative state $z$ and pursuee controls $\eta_1$. 
We denote these estimates by $\hat z := \begin{bmatrix}
        \hat x, &
        \hat y, &
        \hat\theta
    \end{bmatrix}^{\top}, ~
    \hat \eta_1 := \begin{bmatrix}
        \hat v_1, &
        \hat u_1
    \end{bmatrix}^{\top}$,
with internal dynamics and model outputs expressed as
\begin{subequations}
    \begin{align}
        \dot{\hat z} &= f(\hat z, \hat \eta_1, \eta_2) \label{eq:internal_model}\\
        \hat{\mathsf{y}} &= h(\hat z)  \label{eq:model_output}
    \end{align}
    \label{eq:estimate_dynamics}
\end{subequations}
The estimates $\hat{z}$ are called the \textit{internal states} and the dynamics \eqref{eq:internal_model} are regarded as the \textit{internal model} of the pursuer. Notice that since $\eta_2$ is the local control implemented by the pursuer, it enters the internal dynamics~\eqref{eq:internal_model} directly.  
A block diagram of the plant and internal models is provided in Fig.~\ref{fig:block_diagram}.


The estimation problem from the perspective of the pursuer then 
is to obtain the estimates $\hat{z}$ and $\hat{\eta}_1$ so that the difference between the internal model outputs $\hat{\mathsf{y}} = h(\hat{z})$ and observed outputs $\mathsf{y} = h(z)$ is minimized. 

A key distinction between the current problem and classical state observer formulations is that the unknown time-varying controls of the pursuee ($\eta_1 (t)$) are also being estimated here. Classical observers, such as the Luenberger observer, reconstruct unmeasured states from a known system model together with the measured system inputs and outputs \cite{luenberger1966observers}. Similarly, Kalman filtering assumes a prescribed dynamical model for the state evolution and its uncertainty \cite{kalman1960new}. Although unknown controls may be incorporated by augmenting the state, doing so requires introducing additional assumed dynamic models of those controls. 
Such a model may not be available when the pursuee is moving freely and its controls are unknown to the observer.

\begin{figure}[t]
    \centering
    \includegraphics[width=.48\textwidth, trim = {0 10pt 0 0}, clip = true]{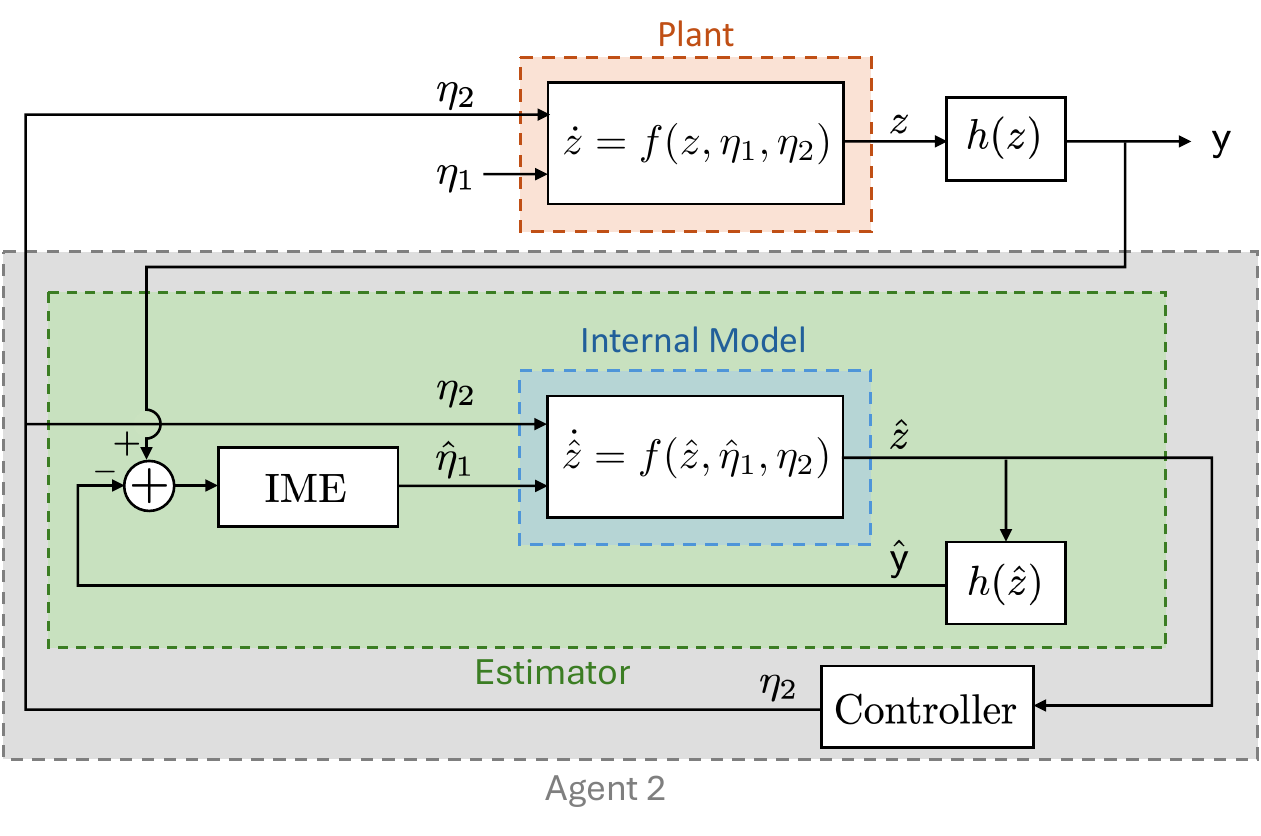}
    \caption{Closed-loop IME architecture from the perspective of agent~2.}
    \vspace*{-15pt}
    \label{fig:block_diagram}
\end{figure}



\subsection{Estimation as an optimization problem}
Consider a finite interval $[t_0,t_N]$ over which observations are available at times $t_0,t_1,\ldots,t_N$. We denote the corresponding data set by $ \mathcal{Y}
    =    \left\{ (t_k,\mathsf{y}_k)\right\}_{k=0}^{N}$, where $  \mathsf{y}_k$ denotes the observation obtained at $t_k$.

Given the observation data set $\mathcal{Y}$, the estimation of the unknown quantities in the internal model is formulated as a dynamic optimization problem, 
where the internal dynamics~\eqref{eq:estimate_dynamics} are imposed as constraints. However, since the observation data arrive at discrete time instants $t_k$, optimal solutions to this optimization problem are known to have jump discontinuities at those times $t = t_k$~\cite{kim2022physics, dey2014control}. 

In order to obtain continuous estimates of the unknown controls $\hat v_1$ and $\hat u_1$ between measurement intervals, we use their time derivatives as the new decision variables, which we denote using
$\hat\omega = [
        \hat\omega_v, ~
        \hat\omega_u
    ]^\top =
    [
        \dot{\hat v}_1, ~
        \dot{\hat u}_1
    ]^\top$.
Thus, for piecewise continuous $\hat\omega$, the reconstructed inputs $\hat v_1$ and $\hat u_1$ remain continuous over the entire estimation interval. Let us denote the set of admissible decision variables as ${\Omega} = \{\hat\omega:[t_0,t_N]\rightarrow{\R}^2 ~:~ \hat\omega ~\text{piecewise continuous}\}$.

The internal model dynamics~\eqref{eq:estimate_dynamics} are then updated as
\begin{align}
\begin{split}
    \dot{\hat z}&= f(\hat z,\hat\eta_1, \eta_2), ~~ \dot{\hat\eta}_1= \hat\omega \\
    \hat{\mathsf{y}} &= h(\hat{z})
\end{split}
\label{eq:estimated_augmented_dynamics}
\end{align}
Finally, for the data set $\mathcal{Y}$, we define the objective function
\begin{equation*}
\begin{aligned}
    J(\tilde{z}_0, \hat\omega;\mathcal{Y}) &= \frac{1}{2}\sum_{k=0}^{N} \norm{\hat{\mathsf{y}}(t_k)-\mathsf{y}_k}^2 + \frac{\chi}{2} \int_{t_0}^{t_N} \norm{\hat{\omega}}^2 \, \dif t 
\end{aligned}
\label{eq:IME_cost}
\end{equation*}
where $\tilde{z} := (\hat z, \hat{\eta}_1)$ is the augmented state with $\tilde{z}_0 = \tilde{z}(t_0)$ considered as an additional decision variable,  and $\norm{\hat{\mathsf{y}}(t_k)-\mathsf{y}_k}^2$ is the observation cost penalizing the disagreement between the internal model outputs and the measurements. The integral term regularizes the rates of change of the estimated leader inputs with a regularization parameter $\chi>0$. 

In summary, the estimation problem is formulated as:
\begin{equation}
\begin{aligned}
    \underset{\tilde{z}_0 \in \R^5,~ \hat\omega (\cdot) \in \Omega}{\text{minimize}}
    \quad& J(\tilde z_0,\hat\omega;\mathcal{Y})\\
    \mathrm{subject\ to}
    \quad& \text{internal dynamics~\eqref{eq:estimated_augmented_dynamics}}
\end{aligned}
\label{eq:optimal_estimation_problem}
\end{equation}
The necessary conditions of optimality for the problem \eqref{eq:optimal_estimation_problem} are obtained using Pontryagin's Maximum Principle (PMP)~\cite{liberzon2011calculus, pontryagin1962mathematical}, which are discussed next. 

\vspace*{-3pt}
\subsection{Necessary optimality conditions}
\label{subsec:pmp}
Let us introduce the costate $p =
    \begin{bmatrix}
        p_x, \; p_y, \; p_\theta, \; p_v, \; p_u
    \end{bmatrix}^\top \in \R^5$
associated with the augmented state $\tilde{z}$, and define the control or pre-Hamiltonian as
\begin{equation*}
    H(\tilde{z},p,\hat\omega, t) = p^{\top} \begin{bmatrix}
        f(\hat z,\hat\eta_1, \eta_2 (t)) \\
        \hat{\omega}
    \end{bmatrix}
     - \frac{\chi}{2} \norm{\hat{\omega}}^2
\label{eq:hamiltonian}
\end{equation*}
We then state the following result.

\smallskip
\begin{proposition} 
Consider the minimization problem \eqref{eq:optimal_estimation_problem}. Suppose $\hat\omega^\star$ is a minimizer and $\tilde{z}^\star = (\hat z^\star, \hat{\eta}_1^\star)$ is the corresponding augmented state trajectory. Then there exists a costate trajectory $p:[t_0,t_N]\rightarrow\R^5, p \not\equiv 0$, such that 
\begin{equation}
\small
    \dot{\tilde{z}}^\star = \nabla_p H(\tilde z^\star, p,\hat\omega^\star, t), 
    ~~
    \dot p = - \nabla_{\tilde{z}} H(\tilde z^\star, p,\hat\omega^\star, t)  
\label{eq:hamilton_equations}
\end{equation}
with jump and boundary conditions 
\begin{subequations}
\small
\begin{align}
    p(t_k^-) = p(t_k^+) - \tfrac{1}{2} \nabla_{\tilde z} \, \norm{\hat{\mathsf{y}}(\hat z^\star(t_k)) - \mathsf{y}_k}^2,&  ~k=0, \cdots, N
\label{eq:costate_jump} \\
    p(t_N^+)=0,
    \qquad
    p(t_0^-)= 0&
\label{eq:costate_boundary_conditions}
\end{align}
\label{eq:transversality_conditions}
\end{subequations}
The optimal control is obtained through the pointwise maximization of the pre-Hamiltonian
\begin{equation}
\begin{aligned}
    \hat\omega^\star
    =
    \arg\max_{w\in\R^2}
    H(\tilde z,p,w, t)
    =
    \frac{1}{\chi}
    \begin{bmatrix}
        p_v\\
        p_u
    \end{bmatrix}
\end{aligned}
\label{eq:optimal_unknown_inputs}
\end{equation}
\end{proposition}
\smallskip
\begin{proof}
    It is a standard application of PMP, augmented with cost functions at discrete times~\cite[Theorem 2.2]{dey2014control}. 
\end{proof}

\vspace*{-3pt}
\subsection{Numerical algorithms}
\label{subsec:forward_backward}

In this section, the numerical algorithms used for solving the optimization problem~\eqref{eq:optimal_estimation_problem} and implementing it in an online fashion are described.

\subsubsection{Forward-backward Algorithm}
The optimal control problem~\eqref{eq:optimal_estimation_problem} is solved numerically using an iterative forward-backward algorithm~\cite{chang2020energy}. During the $j$-th iteration, suppose the estimates of the optimal controls and initial state are denoted as $\hat{\omega}^{(j)}(t)$ and $\tilde{z}_0^{(j)}$, respectively. Then, the internal augmented state trajectory $\tilde z^{(j)}(t)$ 
is obtained by integrating the state dynamics~\eqref{eq:estimated_augmented_dynamics} forward in time. 
The corresponding costate trajectory $p^{(j)}(t)$ is then obtained by integrating the costate dynamics~\eqref{eq:hamilton_equations} backward in time from the terminal condition at $t = t_N$ \eqref{eq:costate_boundary_conditions}, with the jump conditions \eqref{eq:costate_jump} applied at each observation time~$t_k, k = N, N-1, \cdots, 0$.

The decision variables, control $\hat{\omega} (t)$ and initial state $\hat{z}_0$, are updated in the $(j+1)$-th iteration to satisfy the pointwise maximization of the control Hamiltonian~\eqref{eq:optimal_unknown_inputs} and the costate boundary condition at $t=t_0$~\eqref{eq:costate_boundary_conditions}, respectively, as
\begin{subequations}
    \begin{align}
        \hat\omega^{(j+1)}(t) &= \hat\omega^{(j)}(t) 
    + \zeta_{\hat\omega} \nabla_{\hat\omega} H \left( \tilde z^{(j)}(t), p^{(j)}(t),\hat\omega^{(j)}(t), t \right)
\label{eq:control_update} \\
        \tilde z_0^{(j+1)} &= \tilde z_0^{(j)} + \zeta_{\tilde z_0} p(t_0^-) 
\label{eq:initial_state_update}
    \end{align}
\end{subequations}
where $\zeta_{\hat\omega}, \zeta_{\tilde z_0}>0$ are small step size parameters. Explicit formulas for the update \eqref{eq:control_update}, 
along with other additional details appear in Appendix~\ref{appdx:pmp_details}.



\subsubsection{Online implementation using moving horizon}
The forward-backward algorithm described above provides a numerical solution to the estimation problem. However, the algorithm is iterative and may require many forward and backward sweeps before convergence. This is not a significant limitation for an offline reconstruction problem, but it becomes challenging when we seek to use these estimates for online feedback. Repeatedly solving \eqref{eq:optimal_estimation_problem} from an arbitrary initial guess would therefore be impractical for real-time implementation.

To obtain an online estimator, we restrict the optimization in a moving-horizon. At measurement time $t_k$, the problem is solved using the most recent set of observations $\mathcal{Y}_k=\{(t_j,\mathsf{y}_j)\}_{j=k-N_h}^{k}$ over the interval $[t_{k-N_h},t_k]$, where $N_h$ is number of samples allowed in the considered horizon. When a new observation becomes available, the estimation window is shifted forward by one step and the solution from the preceding window is used to initialize the new solution. The overall closed-loop implementation is summarized in Fig.~\ref{fig:block_diagram}, with pseudo-code implementation shown in Algorithm~\ref{alg:IME_fb}.

\begin{remark}
    The moving horizon approach of the proposed estimator is closely related to model predictive control (MPC) \cite{rakovic2019handbook}. In MPC, an optimal control problem is repeatedly solved over a predictive window using current state estimates, and this window is shifted forward as time evolves. Here, an optimal control problem is repeatedly solved using a receding window using past measurements, and the window is likewise shifted forward as new data becomes available. 
\end{remark}
\begin{algorithm}[t]
    \begin{algorithmic}[1]
        \Require Measurements 
        $(t_k,\mathsf{y}_k)$, known controls $\eta_2$
        \Ensure Estimated internal state $\tilde z$
            \State Update the moving window
            $\mathcal{Y}_k=\{(t_i,\mathsf{y}_i)\}_{i=\max(0,k-N_h)} ^{k}$
            \State Initialize $\tilde z_0^{(0)}$ and $\hat\omega^{(0)}$ using the previous solution 
            \For{$j=0$ to MaxIter}
                \State Integrate $\tilde z^{(j)}$ forward over $\mathcal{Y}_k$ using \eqref{eq:estimated_augmented_dynamics}
                \State Integrate $p^{(j)}$ backward using~\eqref{eq:hamilton_equations} and \eqref{eq:transversality_conditions} 
                \State Update $\hat\omega^{(j)}$ using \eqref{eq:control_update}
                \State Update $\tilde z_0^{(j)}$ using \eqref{eq:initial_state_update}
            \EndFor
        \State Store current solution $\tilde z$ and $\hat{\omega}$ 
        \State \Return $\tilde z$
    \end{algorithmic}
    \caption{Internal Model-based Estimator (IME)}
    \label{alg:IME_fb}
\end{algorithm}
    \vspace*{-5pt}

\section{Numerical and Experimental Results} \label{sec:results}

Numerical simulations and physical experiments are used to quantify both estimation accuracy and control execution of the proposed estimator. The code used for all simulations and experiments is available at \cite{icra2027_code}.

\subsection{Bioinspired feedback controls for pursuit}
While the proposed estimator is applicable to any feedback control laws of the form~\eqref{eq:pursuit_general}, this paper focuses on two specific bioinspired pursuit strategies for the evaluation. A brief introduction to these pursuit strategies is provided next.  

\smallskip
\noindent
{\it Shape variables.}
For consistency with pursuit literature, let us represent the relative state $z$ using equivalent \textit{shape variables} consisting of the inter-agent distance $\rho$ and the relative bearing angles $\alpha_1,\alpha_2$ (Fig.~\ref{fig:formation}(a)): $\rho = \sqrt{x^2+y^2}, \alpha_2 = \tan^{-1}\big(\frac{y}{x}\big)$, and  $\alpha_1 = \alpha_2+\pi-\theta \pmod{2\pi}$. The relative dynamics~\eqref{eq:relative_dynamics} can equivalently be expressed in these coordinates (see e.g., \cite[eq. (3)]{li2026feedback}). Notice that the observation model~\eqref{eq:continuous_output} can be equivalently represented by $[\rho, \, \alpha_2]^\top$.

\subsubsection{Constant Bearing (CB)}
Constant-bearing (CB) pursuit regulates the line-of-sight bearing of the pursuer (agent~2) relative to the pursuee
(agent~1) toward a prescribed constant angle $\phi$ \cite{galloway2013symmetry}. The CB pursuit law is obtained by setting the speed control $v_2$ as a constant and the steering control as
\begin{equation}
   u_2 = u_{CB} = \mu\sin(\alpha_2 - \phi) + \frac{1}{\rho}(v_1 \sin\alpha_1 + v_2\sin\alpha_2)
    \label{eq:control_CB}
\end{equation}
where $\mu>0$ is the feedback gain. A key feature of this law is that it makes the manifold $\mathcal{M}_{\mathrm{CB}}
:=
\{(\rho,\alpha_1,\alpha_2):\alpha_2=\phi\}$ invariant and attractive.

\subsubsection{Mutual Motion Camouflage (MMC)}

Motion camouflage is a pursuit strategy, observed in dragonflies and falcons, in which a pursuer moves so that its bearing from the target remains approximately constant, thereby reducing the apparent transverse motion seen by the target~\cite{glendinning2004mathematics, justh2006steering}. When two agents simultaneously employ motion camouflage while pursuing one another, the resulting coordinated behavior is referred to as mutual motion camouflage
(MMC), a behavior also observed in dragonflies~\cite{mischiati2012dynamics}. 
The pursuit laws for MMC are obtained by setting the speed controls $v_1$ and $v_2$ as constants and the steering controls as
\begin{equation}
    u_1=u_2=u_{\mathrm{MMC}} = -\xi\lambda + k_d\lambda\gamma(E-E_0)
    \label{eq:control_MMC}
\end{equation}
where $\lambda=-(v_1\sin\alpha_1+v_2\sin\alpha_2)$, $\gamma=-(v_1\cos\alpha_1+v_2\cos\alpha_2)$, $E=\rho^2\lambda^2e^{-2\xi\rho}$ is called the energy, and~$\xi, k_d > 0$ are control gains. 
This control law renders the desired energy-level manifold $\mathcal{M}_{\mathrm{MMC}} := \left\{(\rho,\alpha_1,\alpha_2): E=E_0\right\}$ 
invariant and locally attractive~\cite{mischiati2012dynamics}. 

\subsection{Numerical simulation results}
The IME is implemented in MATLAB (\textit{ode45} for integrating all differential equations) using a forward-backward iterative solver for the PMP necessary conditions and the moving horizon approach for online implementation, as outlined in Sec.~\ref{subsec:forward_backward}. 
Relevant parameters are tabulated in Table~\ref{tab:parameters}.

	\begin{table}[h]
		\centering
		\caption{Parameters for IME}
		\hskip-10pt
		\begin{tabular}{clc}
			\hline
			\hline\noalign{\smallskip}
			Parameter & Description & Value \\
			\hline\noalign{\smallskip}
			\multicolumn{3}{c}{Forward-Backward Algorithm}\\\noalign{\smallskip}
    		$\chi$ & Regularization parameter & $2$ \\
    		$\zeta_{\hat\omega},\,\zeta_{\tilde z_0}$
    		& Step sizes & 0.25 \\
    		$\mathrm{MaxIter}$ & Maximum iterations & $10$ \\
			\hline\noalign{\smallskip}
			\multicolumn{3}{c}{Moving Horizon Algorithm}\\\noalign{\smallskip}
            $N_h$ & Horizon length  & $30$ \\
            \hline\noalign{\smallskip}
            \multicolumn{3}{c}{Numerical Simulation \& Experiment}\\\noalign{\smallskip}
    		- 
            & Observation sampling interval [$\mathrm{s}$] & $0.1$ \\

			\hline
		\end{tabular}
        \vspace*{-10pt}
		\label{tab:parameters}
	\end{table}

The proposed IME is compared against an extended Kalman filter (EKF) and a particle filter (PF) \cite{gordon1993particle} under both open-loop estimation and closed-loop pursuit. 
All three estimators reconstruct the relative state $\hat{z}$ and unknown pursuee controls $\hat{\eta}_1$. The EKF and PF use the relative dynamics \eqref{eq:relative_dynamics}, but they assume zero dynamics for $\eta_1$, attempting to estimate it as a relatively constant value. 
For both EKF and PF, zero-mean Gaussian process noise is added with standard deviation $0.05$ [\(\mathrm{m}\)] for \(\rho\), $0.09$ [$\mathrm{rad}$] for both \(\alpha_1\) and \(\alpha_2\), and $2\times10^{-3}$ for both \(v_1\) and \(u_1\), with units of [\(\mathrm{m/s}\)] and [\(\mathrm{rad/s}\)], respectively. Synthetic LiDAR measurements are generated by adding zero-mean Gaussian noise with standard deviations $2.5\times10^{-2}$ [\(\mathrm{m}\)] and \(0.02\) [\(\mathrm{rad}\)] to the ground truth \(\rho\) and \(\alpha_2\). All estimators receive the same measurement noise sequence to keep consistent sensing conditions. 

To compare sensitivity to initialization, the IME and EKF are initialized with randomly drawn initial guesses of the unknown quantities
for 100 trials, while the PF started each trial with a different set of 35,000 particles covering operationally admissible ranges.

\subsubsection{Experiment A} 

The estimators are first evaluated independently of the feedback controller, with agent 2 estimating the relative configuration of agent 1 (${z}$) together with its controls (${\eta}_1$).

The prescribed open-loop trajectory uses
\(v_1(t)=0.13+0.03\sin(0.30t)\) [\(\mathrm{m/s}\)] and
\(u_1(t)=-0.16+0.03\sin(0.10t)\) [rad/s],
while agent 2 is driven by the constant inputs
\(v_2=0.07\) [\(\mathrm{m/s}\)] and
\(u_2=-0.16\) [\(\mathrm{rad/s}\)].

To quantify the error in the relative configuration estimation, we define
the relative-pose error $e_g :=
    \left\|
        I_3-g^{-1}\hat{g}
    \right\|_F 
    \label{eq:relative_pose_error}$
where \(g\) and \(\hat{g}\) represent the true and estimated relative configurations, respectively, using the \(SE(2)\) representation
introduced in \eqref{eq:relative_group_element}, and $\norm{\cdot}_F$ denotes the Frobenius norm. 
The (normalized) translational and angular velocity estimation errors are defined as\footnote{A small $\epsilon >0$ is added to the denominator to avoid ill-posedness of the error normalizations whenever $v_1$ or $u_1$ becomes zero.} $e_{v_1}
    :=
    \left|\tfrac{v_1 - \hat{v}_1}{v_1}\right|,
    $ and $e_{u_1}
    :=
    \left|\tfrac{u_1 - \hat{u}_1}{u_1}\right|$, respectively.

Figure~\ref{fig:open_loop_errors} compares the open-loop estimation errors for the estimators. The IME maintains the lowest median relative pose error $e_g$ and narrower percentile bounds after the initial transient. The EKF shows the greatest sensitivity to the initial guesses, while the PF mantains a smaller spread but a larger median error than the IME.  
The IME maintains the lowest $e_{v_1}$. For $e_{u_1}$, the IME exhibits a large percentile spread, but the error decreases as the estimation window grows and remains lower than that of the EKF and PF after convergence. The oscillatory behavior in the EKF and PF error curve is consistent with the mistmatch between their constant model for $\eta_1$ and the time-varying target inputs. The IME provides the lowest estimation errors after convergence and shows the least sensitivity to initialization. 

\begin{figure}[t]
    \centering
    \includegraphics[width=\columnwidth]{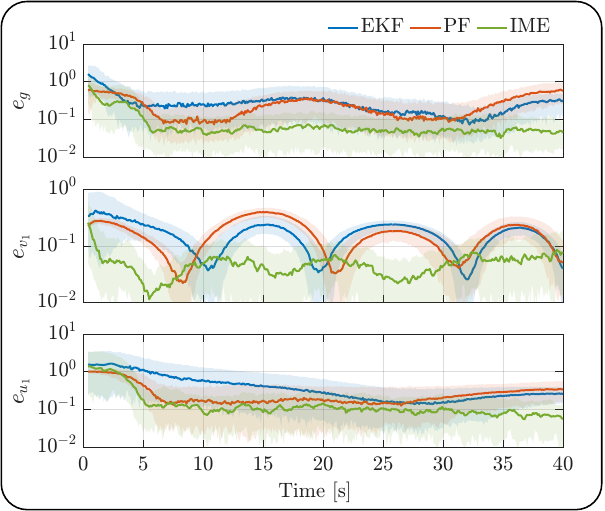}
    \caption{Open-loop estimation performance of the EKF, PF, and IME for experiment A. From top to bottom: \(e_g\), \(e_{v_1}\), and \(e_{u_1}\). Solid curves show medians across initializations; shaded regions show the 10th--90th percentile range.
}
    \vspace*{-15pt}
    \label{fig:open_loop_errors}
\end{figure}

\subsubsection{Experiment B}

We next evaluate the estimators as components of the closed feedback loop
under the constant-bearing pursuit strategy. The feedback law in
\eqref{eq:control_CB} is evaluated using the quantities reconstructed by
each estimator. Consequently,
the EKF, PF and IME each generate their own pursuer control input and
closed-loop trajectory.
The pursuee is driven by the time-varying inputs
$v_1(t)=0.10+0.03\sin(0.30t)$ [\(\mathrm{m/s}\)]
and
$u_1(t)=0.13+0.05\cos(0.30t)$ [\(\mathrm{rad/s}\)],
while the pursuer translational speed is held constant at
$v_2=0.11$ [\(\mathrm{m/s}\)]. The constant-bearing controller uses
\(\mu=2\) and desired bearing \(\phi=20^\circ\).
Each simulation is terminated if the true inter-agent separation reaches $0.25$ [\(\mathrm{m}\)]. 

To quantify achievement of the constant-bearing objective under the feedback, we define $\Lambda_{\mathrm{CB}} (t)
    :=
    \frac{1}{2}
    \left[
        1-\cos(\alpha_2 (t)-\phi)
    \right]$,
where \(\alpha_2\) is the true relative bearing. Thus, 
\(\Lambda_{\mathrm{CB}}=0\) if and only if $\alpha_2 = \phi$ modulo $2\pi$. 

Figure ~\ref{fig:cb_error} compares the resulting closed-loop performance for the three estimators. All three estimators make it possible for the pursuer to quickly reduce the constant-bearing error from its initial value and maintain it near zero throughout the simulation. The IME achieves a lower median constant-bearing error after the initial transient and shows a significantly smaller percentile range over much of the simulation, showing lower sensitivity to initial guesses. The IME also has better reconstruction of the time-varying inputs $v_1$ and~$u_1$. 

\begin{figure}[t]
    \centering
    \includegraphics[width=\columnwidth]{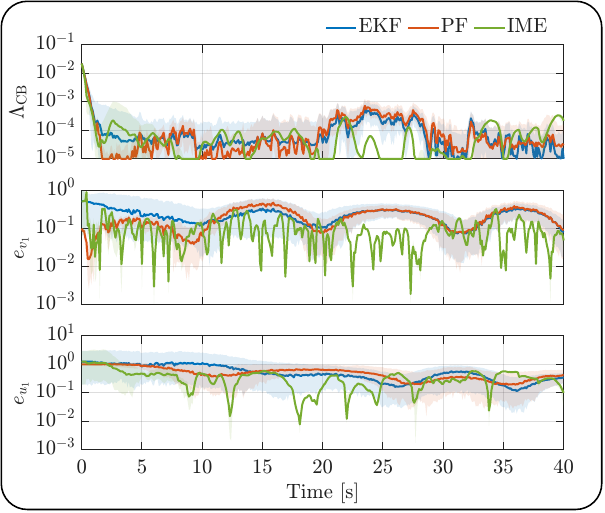}
    \caption{Closed-loop pursuit and estimation performance of the EKF, PF, and IME for experiment B. From top to bottom: \(\Lambda_{\mathrm{CB}}\), \(e_{v_1}\), and \(e_{u_1}\). Solid curves show medians across initializations; shaded regions show the 10th--90th percentile range.
    }
    \vspace*{-15pt}
    \label{fig:cb_error}
\end{figure}


\subsection{Robotic experimental results}

To demonstrate the performance of the IME in closed loop, we implement the constant-bearing (CB) and mutual motion camouflage (MMC) controllers on a pair of TurtleBot3 Burger differential-drive robots~\cite{turtlebot3_manual}. Each robot performs sensing, state estimation, and feedback control locally using its onboard Raspberry Pi. 

Each robot is equipped with two LD-19 LiDARs \cite{youyeetoo_ld19}, mounted at the front and rear and
operating at approximately $10$~Hz, as illustrated in Fig.~\ref{fig:burger}. 
A 3D-printed cylindrical LiDAR identifier is mounted at the geometric center of each robot and serves as the target detected by the LiDARs on the other robot. Because the LiDARs measure within a horizontal scanning plane, the cylindrical geometry provides a consistent cross-sectional return for range and bearing detection.

Each LiDAR provides an effective field of view of approximately
$270^\circ$. The remaining $90^\circ$ sector, directed toward the robot center and its own target pole, is masked in software to prevent self-detection (see Fig.~\ref{fig:burger}). 
The two laser scans are transformed into a common
robot-center frame and merged prior to target-object extraction.
The scan-merging and object-detection stages are implemented using the ROS2 packages in \cite{dual_laser_merger} and
\cite{lidar_object_detection_ros2}, respectively.



Ground truth trajectories are recorded using an OptiTrack motion capture system \cite{optitrack_website} covering approximately $2~[\mathrm{m}]\times2~[\mathrm{m}]$ test area. The motion capture measurements are transmitted through ROS2 and time-aligned with the onboard LiDAR, estimator, and control data. It is to be emphasized that the OptiTrack measurements are used \textit{only} for post-experiment evaluation and are not provided to either the estimator or the controller onboard.

\begin{figure}[!t]
    \centering
    \includegraphics[width=\columnwidth]{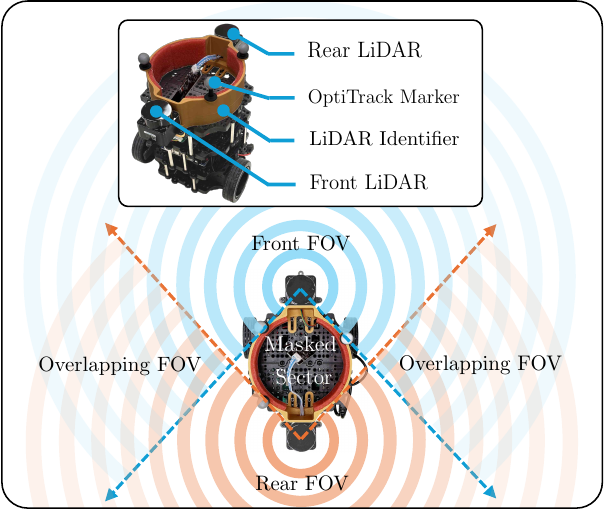}
    \caption{Top view of the TurtleBot3 Burger (side view in the inset) showing the dual LiDAR configuration. The effective fields of view (FOVs) of the front and rear LiDARs are shown in blue and orange, respectively. 
    }
    \vspace*{-20pt}
    \label{fig:burger}
\end{figure}

\subsubsection{Experiment C}
\begin{figure*}[t]
    \centering
    \includegraphics[width=\textwidth, trim = {0 10pt 0 0}, clip = true]{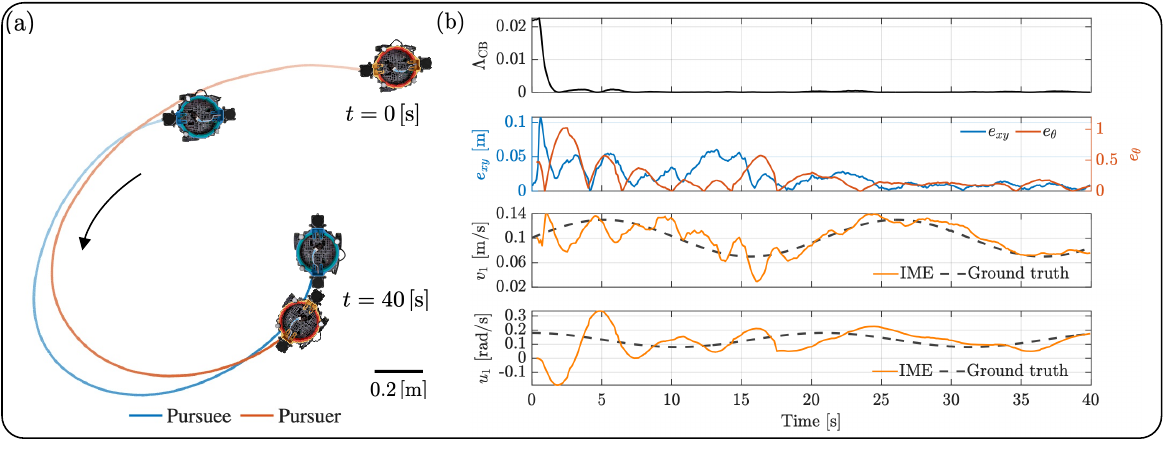}
    \caption{Results for Experiment C. (a) OptiTrack-recorded trajectories of the pursuee (blue) and the pursuer (orange). (b) From top to bottom: constant-bearing error $\Lambda_{\mathrm{CB}}$, geometric relative-pose errors $e_{xy}$ and $e_{\theta}$, pursuee linear speed estimation, and pursuee angular speed estimation. In the $v_1$ and $u_1$ panels, the black dashed lines indicate the commanded inputs of pursuee and the yellow lines indicate the corresponding IME estimates by the pursuer.}
    \vspace*{-15pt}
    \label{fig:exp_cb}
\end{figure*}
To evaluate the performance of IME in a physical closed-loop experiment with continuously varying unknown control inputs, the constant bearing (CB) pursuit strategy is implemented. The robots are initialized at 
$(\rho_0,\alpha_{1,0},\alpha_{2,0})=(1.18,-2.91,0.05)~[\mathrm{m}, \mathrm{rad}, \mathrm{rad}]$.  
Time-varying pursuee controls ($v_1$ and $u_1$) and pursuer linear speed ($v_2$) are kept same as in Experiment B. 
The pursuer steering control ($u_2$) is calculated from~\eqref{eq:control_CB} using the IME estimates. All other parameters are kept the same as in Experiment~B.


Figure~\ref{fig:exp_cb} summarizes the experimental results. As shown in Fig.~\ref{fig:exp_cb}(a), the pursuer manages to intercept the pursuee.
Closed-loop performance and estimator errors are reported in Fig.~\ref{fig:exp_cb}(b). 
Starting from a comparatively large initial
bearing error, the CB performance metric $\Lambda_{\mathrm{CB}}$ decreases to near zero within the
first few seconds and remains close to zero for the remainder of the experiment, indicating that the pursuer rapidly establishes and maintains the desired CB condition.

The geometric estimation error $e_g$ defined in Sec.~\ref{eq:relative_pose_error} is separated into its translational and angular components, $e_{xy}=
\sqrt{(x-\hat x)^2+(y-\hat y)^2}$ and $e_{\theta}=2 \sqrt{1 - \cos(\theta-\hat\theta)}$, which are shown on the left and right vertical axes, respectively, in the second row of Fig.~\ref{fig:exp_cb}(b). The translational error remains relatively small throughout the experiment, whereas the angular component
exhibits a larger error, consistent with the setup of the estimation problem. 



The lower rows of Fig.~\ref{fig:exp_cb}(b) compare the IME estimates $\hat v_1$ and $\hat u_1$ with the corresponding time-varying commands applied to the pursuee. The estimate $\hat v_1$ closely follows the commanded linear speed, while $\hat u_1$ exhibits larger deviations but captures
the overall time-varying trend of commanded steering rate.

\subsubsection{Experiment D}

We further evaluate the proposed estimation framework for the mutual motion
camouflage (MMC) strategy. In contrast to the pursuer-pursuee configuration
considered previously, each agent now performs its own local internal state
estimation and independently generates its steering command using the MMC
controller~\eqref{eq:control_MMC}.
Specifically, agent 1 uses its local LiDAR measurement
$(\rho,\alpha_1)$ to reconstruct $\alpha_2$ and $\eta_2$, whereas agent 2 uses
$(\rho,\alpha_2)$ measurement to reconstruct $\alpha_1$ and $\eta_1$. 


For this experiment, both agents are commanded with the same constant translational speed, $v_1=v_2=0.15$ [\(\mathrm{m/s}\)]. The MMC parameters are selected as $\xi=1.46$ [\(\mathrm{1/m}\)], $k_d=1000$, and 
$E_0=2.84\times10^{-3}$. 

The robot trajectories in Fig.~\ref{fig:mmc}(a) exhibit the characteristic coordinated motion predicted by the theoretical analysis~\cite{mischiati2012dynamics} and previously demonstrated in experiment where global information feedback was available~\cite{halder2015biomimetic}. 
The $(\rho,\gamma)$ phase portrait in Fig.~\ref{fig:mmc}(b) initially follows an outward spiral before approaching the desired orbit corresponding to $E = E_0$.

The closed-loop control performance is quantified by 
$\Lambda_{\text{MMC}} (t)
=
\left|
\tfrac{E_0 - E(t)}{E_0}
\right|$,
where $E(t)$ is the ground truth 
derived from OptiTrack. Starting from approximately $0.8$, $\Lambda_{\text{MMC}}$ decreases and subsequently fluctuates around $0.1$ (Fig.~\ref{fig:mmc}(c)), indicating that the realized closed-loop motion remains near the desired energy level after the initial transient. 
Two characteristic peaks in $\Lambda_{\text{MMC}}$ are observed. The larger peaks occur near the maximum inter-agent separation, while smaller peaks occur near the minimum separation, suggesting that sensing and reconstruction errors vary with the relative geometry.

Finally, since each agent maintains its own observer-centered internal model, Fig.~\ref{fig:mmc}(c) plots the relative-pose errors $e_{g,1}$ and $e_{g, 2}$, each computed between the locally estimated relative configuration and the corresponding OptiTrack-derived ground truth in that agent's body-fixed frame. As shown in Fig.~\ref{fig:mmc}(c), both errors
fluctuate throughout the experiment, but their overall magnitudes remain approximately consistent over the entire experimental duration.
\begin{figure*}[t]
    \centering
    \includegraphics[width=\textwidth]{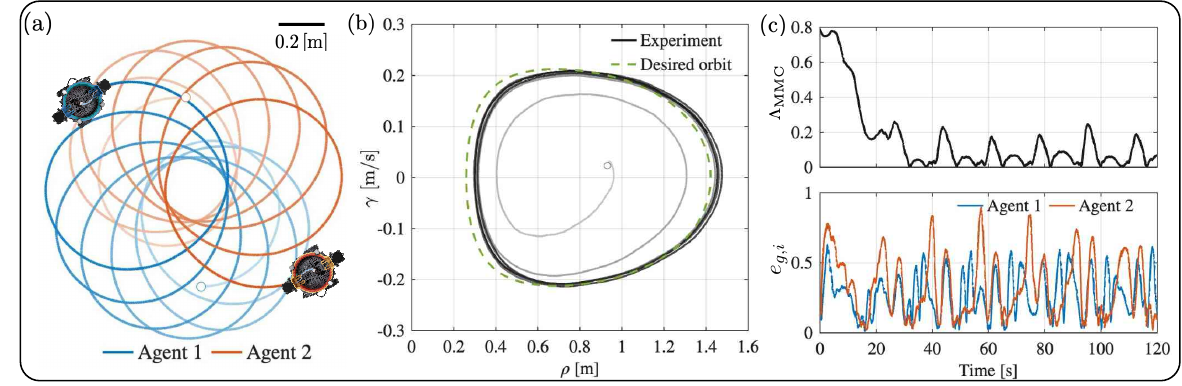}
    \caption{Results for Experiment D. (a) Experimental trajectory using mutual motion camouflage. (b) Phase portrait comparing the experimental trajectory (black) with the desired MMC orbit (green). (c) Geometric relative-pose errors $e_{g,i}$ for Agent~1 (blue) and Agent~2 (orange), and the normalized energy-orbit error $\Lambda_{\mathrm{MMC}}$ computed from OptiTrack ground truth (black).}
    \vspace*{-15pt}
    \label{fig:mmc}
\end{figure*}

\section{Conclusion and future work} \label{sec:conclusion}
This paper develops a bioinspired internal model-based estimator for reconstructing the relative configuration and unknown motion of a planar agent from intermittent range and bearing measurements. The estimator maintains a continuous time model of the relative dynamics and formulates state reconstruction as an optimal control problem. Pontryagin's Maximum Principle is used to derive the necessary optimality conditions, which are solved using a forward-backward algorithm. To enable online estimation, a moving-window approach is used. Numerical simulations demonstrate improved estimation performance relative to conventional estimators, while robotic experiments demonstrate reconstruction of the unmeasured relative heading and target motion from LiDAR measurements. The resulting estimates are further incorporated into feedback laws for constant-bearing pursuit and mutual motion camouflage.

Several directions remain for future work. The nonlinear observability properties of the relative system, including conditions under which the unknown target motion can be uniquely recovered from measurements, will be assessed. The online implementation also motivates further study of computational efficiency, convergence, sensitivity to algorithmic parameters. Extensions to more general target dynamics, noisy measurements, and constraints on the unknown inputs are also of interest. Finally, the framework can be extended to multi-agent and three-dimensional settings.

\vspace{-6pt}
\bibliographystyle{IEEEtran}
\bibliography{bibfiles/halder_papers,bibfiles/reference}

\appendices
\renewcommand{\thelemma}{A-\arabic{section}.\arabic{lemma}}
\renewcommand{\thetheorem}{A-\arabic{section}.\arabic{theorem}}
\renewcommand{\theequation}{A-\arabic{equation}}
\renewcommand{\thedefinition}{A-\arabic{definition}}
\setcounter{lemma}{0}
\setcounter{theorem}{0}
\setcounter{equation}{0}

\vspace*{-5pt}
\section{Details of the PMP Necessary Conditions}
\label{appdx:pmp_details}
The costate dynamics in~\eqref{eq:hamilton_equations} are expanded as
\begin{align}
\dot p_x &= u_2p_y, \quad
\dot p_y = -u_2p_x, \quad
\dot p_\theta = \hat v_1(p_x\sin\hat\theta-p_y\cos\hat\theta)
\nonumber\\
\dot p_v &= -p_x\cos\hat\theta-p_y\sin\hat\theta, \qquad
\dot p_u = -p_\theta
\end{align}
The forward-backward update rule for the decision variables $\hat{\omega}$~\eqref{eq:control_update} is explicitly written as
\begin{align}
\hat\omega^{(j+1)} &= \hat\omega^{(j)} + \zeta_{\hat\omega} \left(
-\chi\hat\omega^{(j)} +
\begin{bmatrix}
p_v^{(j)}\\
p_u^{(j)}
\end{bmatrix}\right)
\end{align}
The only \textit{non-zero} costate jump conditions in~\eqref{eq:costate_jump} are
\begin{align}
    \begin{bmatrix}
        p_x \\ p_y
    \end{bmatrix} (t_k^-) = 
    \begin{bmatrix}
        p_x \\ p_y
    \end{bmatrix} (t_k^+) -
    \begin{bmatrix}
        \hat{x} - x \\ \hat{y} - y
    \end{bmatrix}, ~ k = 0, 1, \cdots N 
\end{align}
Notice that the error between the estimates and measurements enters the whole framework through these jump conditions.

\end{document}